\documentclass[letterpaper]{article} 
\usepackage{aaai23}  
\usepackage{times}  
\usepackage{helvet}  
\usepackage{courier}  
\usepackage[hyphens]{url}  
\usepackage{graphicx} 
\usepackage{natbib}  
\usepackage{caption} 
\usepackage{algorithm}
\usepackage{algorithmic}

\usepackage{stfloats}
\usepackage{float}

\usepackage{amsmath}
\usepackage{amssymb}
\usepackage{mathtools}
\usepackage{amsthm}

\usepackage{color}
\usepackage{multirow}
\usepackage{graphicx}
\usepackage{subfigure}

\usepackage{newfloat}
\usepackage{listings}
\DeclareCaptionStyle{ruled}{labelfont=normalfont,labelsep=colon,strut=off} 
\floatstyle{ruled}
\newfloat{listing}{tb}{lst}{}
\floatname{listing}{Listing}
\theoremstyle{plain}
\newtheorem{theorem}{Theorem}
\newtheorem{proposition}{Proposition}
\newtheorem{lemma}{Lemma}

\theoremstyle{definition}
\newtheorem{definition}{Definition}

\theoremstyle{remark}

\usepackage{smile}

\title{Flow to Control: Offline Reinforcement Learning with Lossless Primitive Discovery}
\author{
    Yiqin Yang\textsuperscript{\rm 1}\equalcontrib, Hao Hu\textsuperscript{\rm 2}\equalcontrib, Wenzhe Li\textsuperscript{\rm 2}\equalcontrib, Siyuan Li\textsuperscript{\rm 3}\\
    Jun Yang\textsuperscript{\rm 1}, Qianchuan Zhao\textsuperscript{\rm 1}\thanks{Equal advising.}, Chongjie Zhang\textsuperscript{\rm 2}\footnotemark[2]\\
}
\affiliations{
    \textsuperscript{\rm 1}Department of Automation, Tsinghua University\\
    \textsuperscript{\rm 2}Institute for Interdisciplinary Information Sciences, Tsinghua University\\
    \textsuperscript{\rm 3}Harbin Institute of Technology\\

    \{yangyiqi19,huh22,lwz21\}@mails.tsinghua.edu.cn \\

}

\usepackage{bibentry}
\definecolor{amber}{rgb}{1.0, 0.75, 0.0}

\begin{document}

\maketitle

\begin{abstract}


Offline reinforcement learning (RL) enables the agent to effectively learn from logged data, which significantly extends the applicability of RL algorithms in real-world scenarios where exploration can be expensive or unsafe. 
Previous works have shown that extracting primitive skills from the recurring and temporally extended structures in the logged data yields better learning.
However, these methods suffer greatly when the primitives have limited representation ability to recover the original policy space, especially in offline settings.
In this paper, we give a quantitative characterization of the performance of offline hierarchical learning and highlight the importance of learning lossless primitives. To this end, we propose to use a \emph{flow}-based structure as the representation for low-level policies. This allows us to represent the behaviors in the dataset faithfully while keeping the expression ability to recover the whole policy space. We show that such lossless primitives can drastically improve the performance of hierarchical policies. The experimental results and extensive ablation studies on the standard D4RL benchmark show that our method has a good representation ability for policies and achieves superior performance in most tasks.


\end{abstract}

\section{Introduction}
Online reinforcement learning has succeeded dramatically in various domains, including strategy games~\cite{ye2020mastering}, recommendation systems~\cite{swaminathan2015batch}, and continuous control~\cite{lillicrap2015continuous}.
However, it requires extensive interaction with the environment to learn through trial and error.
In many real-world problems, like robot learning and autonomous driving, access to an interactive environment can be severely limited due to safety concerns or huge costs~\cite{shalev2016safe, singh2020cog}. However, in these applications, a large amount of pre-collected data is usually available, including human demonstrations and data from hand-engineered policies. This makes offline RL~\cite{levine2020offline, fu2020d4rl} an appealing approach for effectively learning from such previously logged data.

Leveraging the temporally extended and recurring structure has long been an effective paradigm to solve complex and diverse tasks~\cite{dietterich1998maxq,sutton1999between,kulkarni2016hierarchical} by facilitating exploration~\cite{pertsch2020accelerating}. 
For example, a kitchen robot may need to finish many sub-tasks to cook a dish: slide cabinet door, move kettle, and open microwave~\cite{gupta2019relay}. 
However, such structure in the agent's logged behavior is 
less appreciated
in the offline setting.
Especially, there is little theoretical justification for extracting low-level skills and learning a high-level policy with offline datasets, and it is unclear how offline-learned skills affect the overall performance.
This naturally leads to the following question:

\begin{center}
    
    
    {{\it Is there provable benefit from extracting temporally extended primitive skills in the offline setting, and how to extract useful skills to maximize such benefit?}}
\end{center}

In this paper, we give an affirmative answer to the first question by analyzing the performance of hierarchical offline learning in linear MDPs.
We show that there are provable benefits in learning a hierarchical structure when the low-level skills are well-learned, in the sense that such skills can faithfully represent the logged behavior and be expressive to recover the original policy space. 
As for the second question, we find that current skill-based offline RL methods are significantly compromised by the limited representation ability of the learned low-level skills to recover the original policy space.
Such loss of representation capacity can greatly harm some powerful offline algorithms' optimization, rendering skill-based methods less impractical in general usage.
To recover the policy space in a lossless way and learn more effective skills,
we propose to learn an \emph{invertible} function to map latent embeddings to temporally extended actions.
Since the mapping is invertible, the RL agent retains full control over the original policy space while faithfully representing the original dataset's behavior.
We name our method offline \underline{L}ossless \underline{P}rimitive \underline{D}iscovery~(LPD) and evaluate LPD on the standard D4RL benchmark.
Experimental results indicate that LPD achieves superior performance on challenging tasks.
Extensive ablation studies demonstrate the strong representation ability of the LPD.
To the best of our knowledge, our work is the first study analyzing the representation error in offline hierarchical RL.

\begin{figure*}[t]
    \centering \includegraphics[scale=0.55]{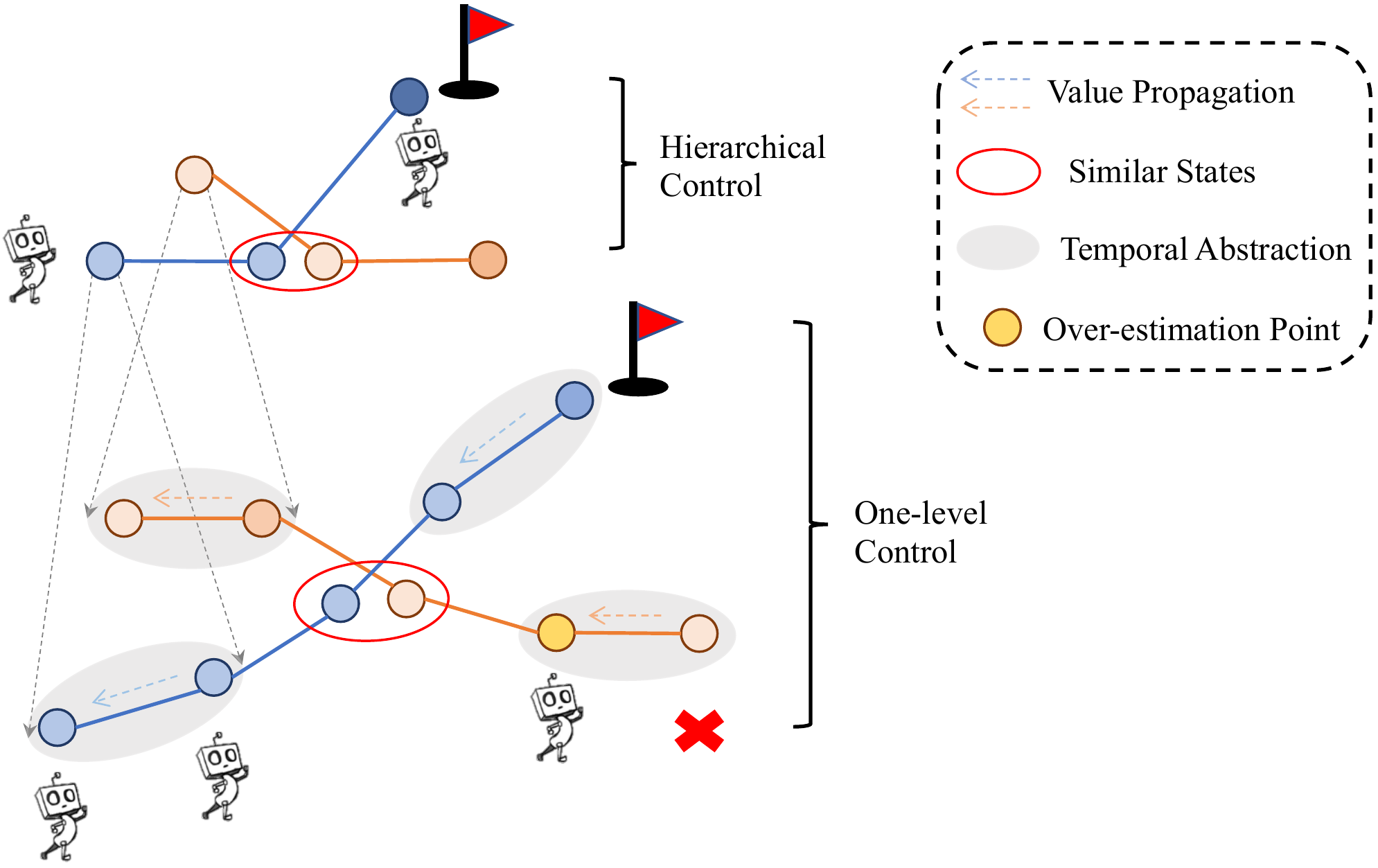}
    \caption{The general description of offline hierarchical learning. The agent makes a decision for every $c$ step and then follows the behavior-cloned primitive skills for $c$ steps. The temporally extended primitive skills help the agent to reduce the estimation error and make the right decision by reducing the planning horizon.
    }
    \label{fig:motivation}
\end{figure*}

\section{Related Work}
\paragraph{Offline RL.} 
Current offline RL methods mainly address the extrapolation error issue in value estimation, and they can be roughly divided into policy constraint~\cite{yang2021believe,peng2019advantage,nair2020accelerating,ma2021offline}, pessimistic value estimation~\cite{kumar2020conservative,ma2021conservative}, and model-based methods~\cite{yu2021combo,kidambi2020morel, hu2022role}.
Policy constraint methods aim to keep the policy close to the behavior under a probabilistic distance.
Pessimistic value estimation methods enforce a regularization constraint on the critic loss to penalize over-generalization.
Model-based methods attempt to learn a model from offline data with minimal modification to the policy learning~\cite{argenson2020model}.
Differently, we focus on the offline dataset's temporal structure
and explore how to improve these modern offline methods by extracting temporally extended primitive behaviors.

\paragraph{Hierarchical learning.}
Learning primitive skills is closely related to hierarchical models~\cite{dietterich1998maxq,sutton1999between,kulkarni2016hierarchical}: the low-level policy is the primitive policy extracted from the offline datasets, while the high-level policy is trained using modern offline RL.
This structure is similar to online skill discovery works~\cite{sharma2019dynamics,eysenbach2018diversity,shankar2020learning,singh2020parrot}, which use unsupervised objectives to discover skills and use the discovered skills for planning~\cite{sharma2019dynamics}, few-shot imitation learning~\cite{nam2022skill}, or online RL~\cite{nachum2018near}.
In contrast to these works designed for online settings, we focus on settings where a large dataset of diverse behaviors is provided, but access to the environment is restricted.
There are some variants of the above works in offline RL. 
For example, OPAL~\cite{ajay2020opal} explores hierarchical offline learning by adopting VAE-based primitive policies and achieves sound performance.
In the following sections, we briefly name the hierarchical offline RL method incorporated with OPAL as x+OPAL (e.g., IQL+OPAL).

\section{Preliminaries}
We consider infinite-horizon discounted Markov Decision Processes (MDPs), defined by the tuple $(\mathcal{S},\mathcal{A},\mathcal{P},r,\gamma),$ where $\mathcal{S}$ is a state space, $\mathcal{A}$ is an action space, $\gamma \in [0,1)$ is the discount factor and $\mathcal{P}: \mathcal{S}\times\mathcal{A}\rightarrow \Delta(\mathcal{S}), r:\mathcal{S}\times\mathcal{A}\rightarrow [0, r_{\text{max}}]$ are the transition function and reward function, respectively. We also assume a fixed distribution $\mu_0 \in \Delta(\mathcal{S})$ as the initial state distribution. 
The goal of an RL agent is to learn a policy $\pi: \mathcal{S}\rightarrow \Delta{(\mathcal{A})}$ under dataset $\mathcal{D}$, which maximizes the expectation of a discounted cumulative reward: $J(\pi)=\mathbb{E}_{\mu_0,\pi}\left[\sum_{t=0}^{\infty}\gamma^t r(s_t,a_t)\right]$. We assume the dataset is generated by an unknown behavior policy $\beta(a|s,z)$, with prior $z\sim Z$. Note that this assumption on data collection is implicit and we do not assume an additional structure of the dataset.
We define the dataset for learning primitives as $\mathcal{D}_{\text{low}}\doteq\{\tau_i=(s_t^i, a_t^i)_{t=0}^{c-1}\}_{i=1}^{N}$, where $\tau_i$ is the sub-trajectory and $c$ is the skill length.
We also create a dataset for high-level policy learning as $\mathcal{D}_{\text{hi}}=\{(s_0^i,z_i,\sum_{t=0}^{c-1}\gamma^tr_t^{i},s_c^i)\}_{i=1}^{N}$, where $z_i$ are learned labels for each sub-trajectory.
When learning from the low-level dataset, we consider a finite candidate function class $\Pi_\theta$. Further, we assume $\beta\in \Pi_\theta$, i.e., the true low-level policy is realizable in the function class $\Pi_\theta$.

To make things more concrete, we consider the \textit{linear MDP}~\citep{yang2019sample,jin2020provably} as follows, where the transition kernel and expected reward function are linear with respect to a feature map.

\begin{definition}[Linear MDP]
    We say an episodic MDP $(\cS,\cA,\cP,r, \gamma)$ is a linear MDP with known feature map $\Psi:\cS\times \cA\times \cS\to \RR^d, \Phi:\cS\times \cA\to \RR^d$ if there exist an unknown vector $\omega \in \RR^d$ such that
    \#
    \cP(s'\given s,a) = \Psi(s,a,s')^\top \omega ,~ \EE\bigl[r(s, a)\bigr] = \Phi(s,a)^\top\omega
    \#
    for all $(s,a,s')\in \cS\times \cA\times \cS$. And we assume $\|\Phi(s,a)\|_\infty\leq 1,\|\Psi(s,a,s')\|_\infty\leq 1$ for all $(s,a,s')\in \cS\times \cA \times\cS$ and $\|\omega\|_2\leq \sqrt{d}$.
    \label{assump:linear_mdp}
\end{definition}

We also consider the every-$c$-step hyper-MDP, where we act every $c$ step to determine the next primitive $z$ and keep the reward and the dynamics the same. Then we have the following proposition (See Appendix~\ref{proofs} for proof).

\begin{proposition}
    \label{prop:1}
    The every-$c$-step hyper-MDP induced by the linear MDP $\cM$ and the primitive policy $\beta(\cdot|s,z)$ is also a linear MDP.
\end{proposition}

For the theoretical analysis, we use the \textit{pessimistic value iteration}~(PEVI) for high-level policy learning in the hyper-MDP. Please see Algorithm~\ref{alg:1} in Appendix~\ref{offline_algorithm} for more details.
The learned primitive skills $\pi_{\theta}(a|s,z)$ and the policy learned from PEVI together comprise a hierarchical policy $\widehat{\pi}_{\theta}$.
Similarly, We use $\widehat{\pi}_{\beta}$ to denote the composition of the high-level policy $\hat{\pi}$ from PEVI and the ground-truth low-level policy $\beta$.
$\pi^*_{\beta}$ is the optimal policy in the hyper-MDP with $\beta$ as primitive skills and $\pi^*$ is the optimal policy in $\cM$. 
We consider the suboptimality as the evaluation metric, which is defined as the performance difference between the optimal policy $\pi^*$ and the given policy $\pi$. Formally, we have $\text{\rm SubOpt}(\pi) = J(\pi^*)-J(\pi).$
    
\begin{figure*}[t]
    \centering
    \subfigure[IQL+OPAL~($c=1$)]{\label{OPAL_visual}\includegraphics[scale=0.35]{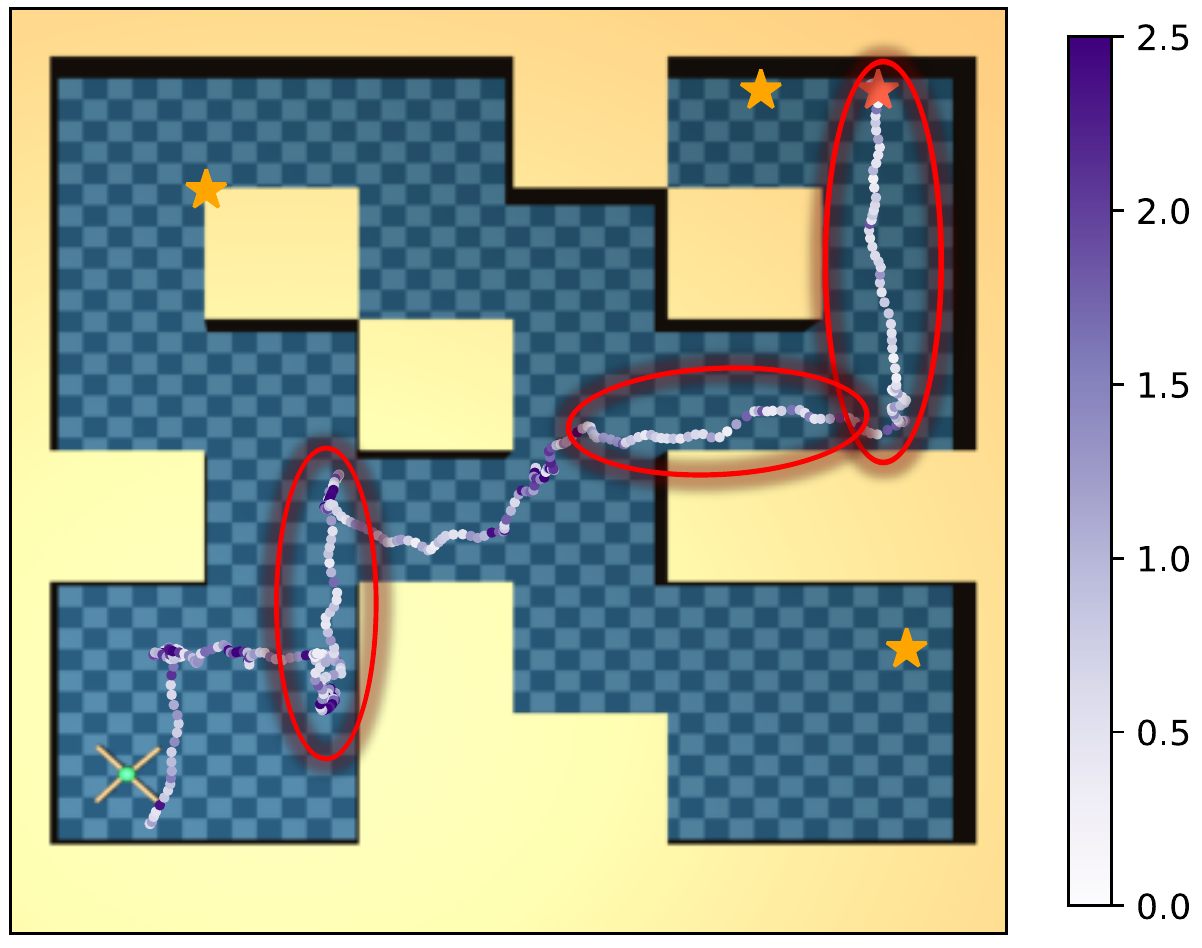}}
    \hspace{0.3cm}
    \subfigure[IQL+LPD~($c=1$)]{\label{HORL_visual}\includegraphics[scale=0.35]{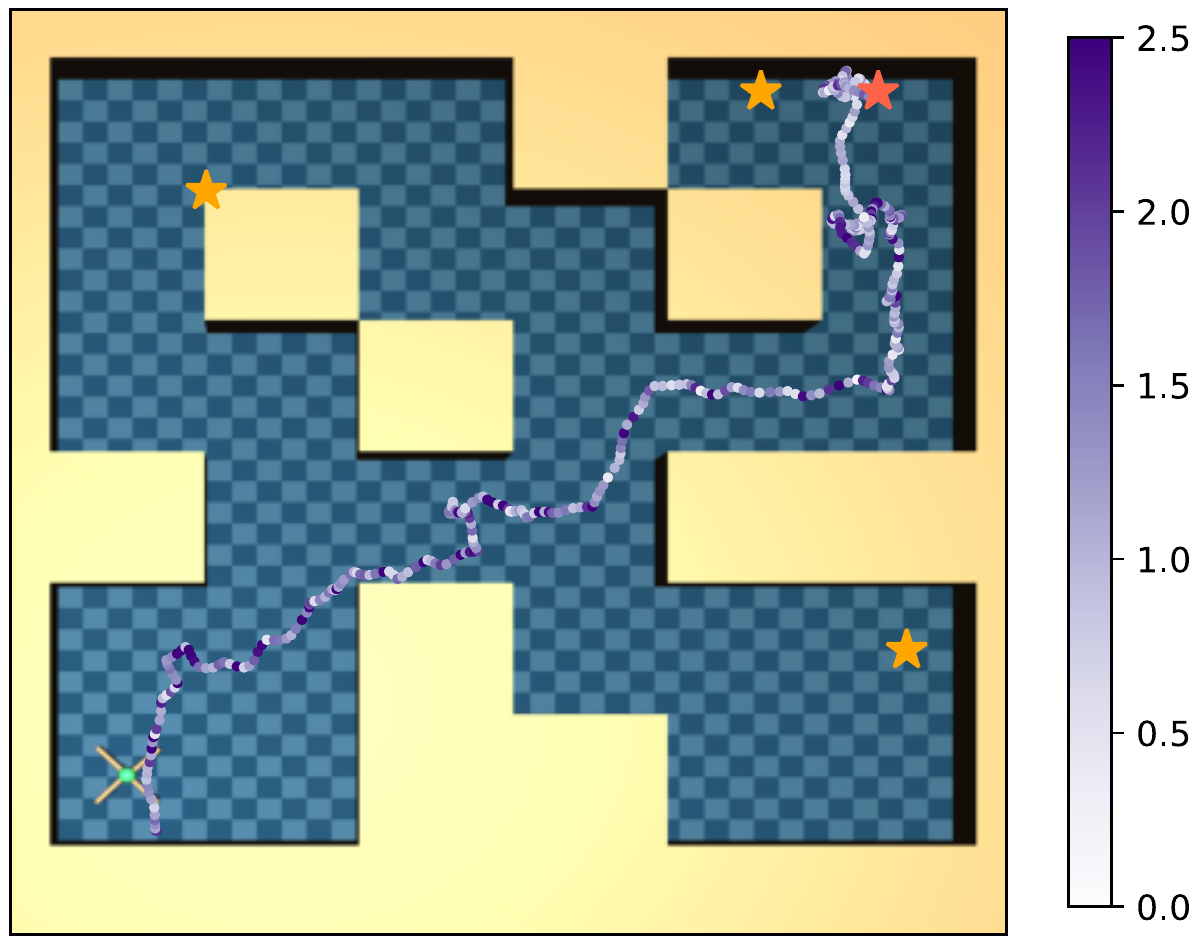}}
    \hspace{0.3cm}
    \subfigure[Performance]{\label{perf_c_1}\includegraphics[scale=0.31]{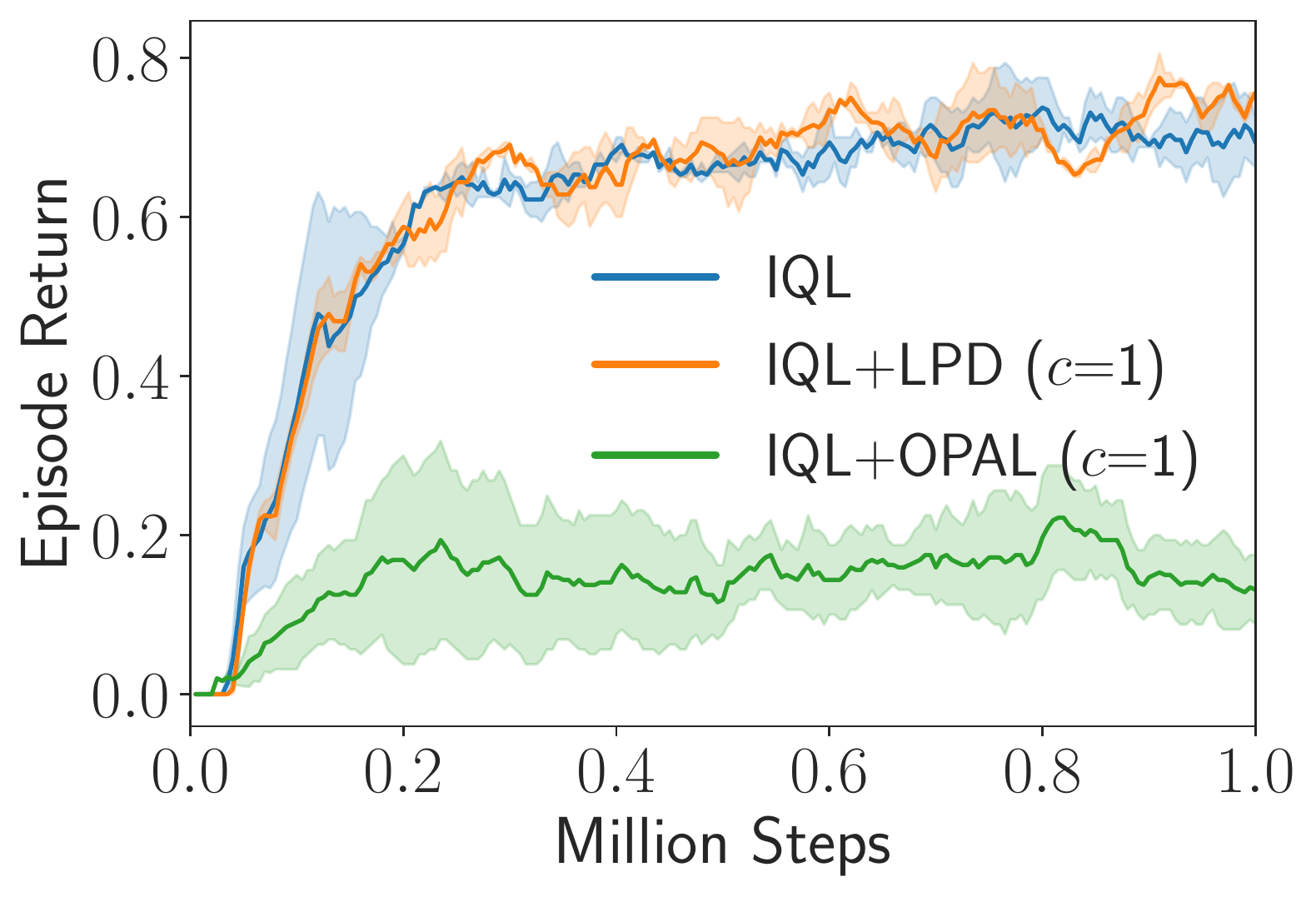}}
    \caption{Impact of limited representation on performance.
    (a) and (b) visualization of the similarity $\epsilon$ between the agent's decision $a$ and datasets $\mathcal{D}$ in the antmaze medium task, which is calculated by $\epsilon=\min_{\hat{a}\in\mathcal{C}}\|a-\hat{a}\|_{1}$, where 
    $\mathcal{C}=\{(\hat{s},\hat{a})\mid\|s-\hat{s}\|_{1}\leq 20, (\hat{s},\hat{a})\in\mathcal{D}$\} is the similar state set to the agent.
    Darker colors correspond to the lower similarity.
    We find the decoded actions of OPAL are limited to the datasets (red circle), which is opposed to the LPD.
    (c) Performance on the antmaze-medium-diverse-v0 task with skill length $c=1$.
    }
    \label{fig:generalization_figure}
\end{figure*}

\section{Theoretical Analysis}

In this section, we analyze the benefit of learning a hierarchical structure in the context of offline RL and linear MDPs.
To derive a performance bound and examine how a primitive learning method affects the overall performance, we consider the following decomposition:
    \begin{align*}
        \text{\rm SubOpt}(\widehat{\pi}_{\theta}) =   \underbrace{J(\widehat{\pi}_{\beta})-J(\widehat{\pi}_{\theta})}_{\text{Primitive Error}}+\\
        \underbrace{J(\pi^*_{\beta})-J(\widehat{\pi}_{\beta})}_{\text{Offline Error}} +\underbrace{J(\pi^*)-J(\pi^*_{\beta})}_{\text{Representation Error}}.
    \end{align*}
The primitive error comes from the generalization error of a learned low-level policy $\pi_\theta$ and the ground-truth low-level policy $\beta$. 
The offline error comes from learning in a high-level offline dataset.
The representation error comes from the limited representation ability of the hierarchical structure. 
In the following, we upper bound the three different kinds of error, respectively.

The primitive error comes from limited samples of the low-level dataset, which makes the learned low-level policy different from the original primitive. This can be seen as a standard machine learning problem and can be characterized as follows.
\begin{lemma}
    \label{lemma:1}
     With high probability $1-\delta$,
    \begin{equation}
        J(\widehat{\pi}_{\beta})-J(\widehat{\pi}_{\theta}) \leq \frac{\gamma c(c+1) r_{\text{max}} }{(1-\gamma)(1-\gamma^c)} \varepsilon_{\theta},
    \end{equation}
    where $\varepsilon_{\theta} = \sqrt{\frac{\ln{|\Pi_\theta|/{\delta}}}{N}}$.
    $|\Pi_\theta|$ is the size of the policy class and $N$ is the number of training samples.
\end{lemma}

On the contrary, the offline error comes from limited samples of the high-level dataset. With proper pessimism~\citep{jin2020provably}, we have the following offline-learning bound in the hyper-MDP. 
\begin{lemma}[Informal]
    \label{lemma:2}
    Suppose there exists an finite concentration coefficient $c^\dagger$ w.r.t. the optimal policy, then with probability $1-\delta$, the policy $\widehat{\pi}$ generated by PEVI satisfies
    \begin{align*}\label{eq:event_opt_explore_d}
        J(\pi^*_{\beta})-J(\widehat{\pi}_{\beta})
    \leq  \frac{2C r_{\text{\rm max}}}{(1-\gamma)(1-\gamma^c)} \sqrt{\frac{c^\dagger d^3\zeta}{N}}, \ \forall s\in \cS,
    \end{align*}
    where $C$ is a constant, $d$ is the dimension of the linear MDP, $N$ is the size of the dataset and $\zeta = \log{(4dN/(1-\gamma)\delta)}$.
    Please refer to the Appendix for the Formal description.
\end{lemma}

Finally, we bound the representation error by quantifying how many of the possible low-level policies can be recovered by different choices of $z$. Following~\citet{nachum2018near}, we make use of an auxiliary inverse primitive model $\Omega(s, a)$, which aims to predict which primitive $z$ will cause $\beta$ to yield an action $\tilde{a}=\beta(s,z)$ that induces a next
state distribution $P(s'|s,\tilde{a})$ similar to $P(s'|s,a)$. Formally, we have
\begin{lemma}
    \label{lemma:3}
     Consider a mapping $\Omega:S\times \Pi \rightarrow Z$ and let $\tilde{\beta} = \beta(s_t,\Omega(s,\pi))$. We define $\varepsilon_k : S\times\Pi \rightarrow \RR$ for $k \in [1,c]$ as 
    \begin{equation}
        \varepsilon_k(s_t,\pi) = \TV (P_\pi(s_{t+k}|s_{t+k-1}) || P_{\tilde{\beta}}(s_{t+k}|s_{t+k-1}) ),\notag
    \end{equation}
    If
    \begin{equation}
        \max_{\pi\in \Pi_\theta, k\in [1,c]} \EE_\pi[\varepsilon_k(s,\pi)] \leq \varepsilon_{\Omega},\notag
    \end{equation}
    Then we have 
    \begin{equation}
        J(\pi^*) - J(\pi^*_{\beta}) \leq \frac{\gamma c(c+1) r_{\text{max}} }{(1-\gamma)(1-\gamma^c)} \varepsilon_{\Omega}.
    \end{equation}
\end{lemma}

Together, we have the following theorem.

\begin{theorem}
    \label{theorem:1}
    Under the condition in Lemma~\ref{lemma:1}, \ref{lemma:2} and~\ref{lemma:3}, the suboptimality of a policy learned in the hyper-MDP with Algorithm~\ref{alg:1} satisfies
    \begin{align}
        \text{\rm SubOpt}(\widehat{\pi}_{\theta}) \leq   \frac{2C r_{\text{\rm max}}}{(1-\gamma)(1-\gamma^c)} \sqrt{\frac{c^\dagger d^3\zeta}{N}}+\notag\\
        \frac{\gamma c(c+1) r_{\text{max}} }{(1-\gamma)(1-\gamma^c)} (\varepsilon_{\Omega}+\varepsilon_{\theta}),
    \end{align}
    with high probability $1-2\delta$.
\end{theorem}

The proofs for all the above lemmas and theorems are provided in Appendix~\ref{proofs}. From Theorem~\ref{theorem:1}, we can see that there is a provable benefit in leveraging the temporal structure in the offline dataset, since a large skill length $c$ effectively reduces the offline error by a factor of $\frac{1-\gamma}{1-\gamma^c}$. However, a large skill length also incurs large representation errors and primitive errors, which forms a trade-off with different choice of $c$. We can see that it is crucial to choose a skill learning method that can faithfully represent the original behavior policy while keeping the policy representation ability. It is also important to choose a suitable skill length $c$ to balance the two errors. Theorem~\ref{theorem:1} gives a qualitative guidance in choosing the proper length. When the dataset has a bad coverage $c^\dagger$ or a large dimension $d$, we may use a larger skill length.
The general description of offline hierarchical learning is shown in Figure~\ref{fig:motivation}.

\begin{figure*}[t]
    \centering
    \includegraphics[scale=0.55]{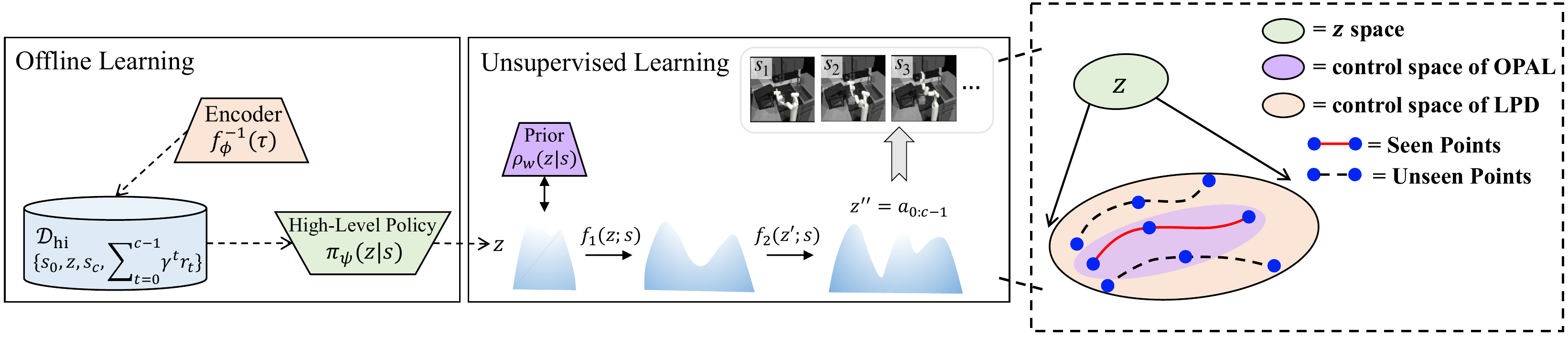}
    \caption{The framework of LPD.
    We first train the invertible primitive policies $f_{\psi}$ by the unsupervised learning. 
    Next, we use $f_{\psi}^{-1}$ to create Dataset $\mathcal{D}_{\text{hi}}$ and further train high-level policy $\pi_{\phi}$.
    LPD-based primitive policies have adequate control over the original policy space due to their strong representation ability.
    On the contrary, the control space of the OPAL-based primitives is limited to the vicinity of the dataset~(\emph{seen points}).}
    \label{fig:HORL}
\end{figure*}

    

\section{Method}
Following the above theoretical analysis, we find that current skill-based offline RL methods are significantly clipped by the limited representation ability of low-level learned skills to recover the original policy space.
Thus, we propose to learn a lossless primitive discovery, which helps the RL agent retain full control over the original policy space while faithfully representing the original dataset's behavior.


\subsection{Lossless Primitive Discovery}\label{limited_repre_issue}
We would like to extract a continuous space of temporally-extended primitives $\pi_{\theta}(a|s,z)$ from $\mathcal{D}_{\text{low}}$ which we can later use as an action space for learning downstream tasks with offline RL.
A common practice~\cite{ajay2020opal} adopts the following objective for learning $\pi_{\theta}$, which maximizes the evidence lower bound (ELBO):

\begin{align*}
    \max_{\theta,\psi,w} J(\theta,\phi,w)&=\hat{\mathbb{E}}_{\tau\sim\mathcal{\mathcal{D}_{\text{low}}},z\sim q_{\psi}(z|\tau)}
    \left[\sum_{t=0}^{c-1}\log \pi_{\theta}(a_t|s_t,z) \right.\\
    &\left.-\text{D}_{\text{KL}}(q_{\psi}(z|\tau)||\rho_{w}(z|s_0))\right]
\end{align*}

where $q_{\psi}(z|\tau)$ denotes the encoder, which encodes the trajectory $\tau$ of state-action pairs into distribution in latent space;
$\pi_{\theta}(a|s,z)$ denotes the decoder, which maximizes the conditional log-likelihood of actions in $\tau$ given the state and the latent vector;
$\rho_{w}(z|s_0)$ denotes the prior, which predicts the encoded distribution of the sub-trajectory $\tau$ from its initial state.
The KL-constraint enforces the distribution $q_{\psi}(z|\tau)$ to be close to just predicting the latent variable $z$ given the initial state of this sub-trajectory.
The conditioning on the initial state regularizes the distribution $q_{\psi}(z|\tau)$ to be consistent over the entire sub-trajectory.

Following the analysis of Theorem~\ref{theorem:1}, we want to minimize the representation error in addition to the primitive error in the low-level policy learning.
However, simple VAE-based models can lead to enormous  representation errors, as verified in the experiment. 
To solve this issue, we propose to learn an invertible mapping between skill space and original policy space $f_\psi:\mathcal{Z}\times\mathcal{S}\rightarrow\mathcal{A}^c$ to replace $q_{\psi}(z|\tau)$, which guarantees our primitive policy is lossless over the original policy space, and the overall objective can be written as

\begin{align*}
    &\min_{\psi,w} J(\psi,w)=\hat{\mathbb{E}}_{\tau\sim\mathcal{D}_{\text{low}}}\left[-\log p_z(f_{\psi}^{-1}(a_{0:c-1};s_0))\right.\\
    &\qquad \qquad \qquad \left.-\log \left|\det\left(\frac{\partial p_z(f_{\psi}^{-1}(a_{0:c-1};s_0))}{\partial a_{0:c-1}}\right)\right|
    \right] \\
    &\text{s.t.}\quad \hat{\mathbb{E}}_{\tau\sim\mathcal{D}_{\text{low}}}\left[\text{D}_{\text{KL}}(f^{-1}_{\psi}(a_{0:c-1};s_0)||\rho_{w}(z|s_0))\right] \leq \epsilon_{\text{KL}}
\end{align*}

where $f_{\psi}$ are parameterized by $\psi$ and $f_\psi^{-1}:\mathcal{A}^c\times\mathcal{S}\rightarrow\mathcal{Z} $ is the partial inverse of $f_\psi$.
The $a_{0:c-1}=f_{\psi}(z;s_{0})$ denotes the action sequence generated by $f_{\psi}$ when given the latent variable $z$ and state $s_0$, and vice versa.
$\det\left(\frac{\partial f(x)}{\partial x^{T}}\right)$ represents the Jacobian determinant.
Based on the works on normalizing flow~\cite{dinh2016density,dinh2014nice}, we can easily represent the invertible function with a neural network and deal with the Jacobian term neatly. Note that our design is general and not limited to \textit{flow}-based structures. Other designs to represent such an invertible $f$ for primitive skills can be an interesting avenue for future work.
The framework of flow-based structure is shown in Figure~\ref{fig:framework_lpd}.
Please refer to Appendix for the implementation details.

\subsection{High-Level Offline Policy Learning}
As for the high-level offline policy optimization, we would like to use the learned invertible mapping $f_{\psi}(z;s_0)$ and prior $\rho(z|s_0)$ to improve offline RL for downstream tasks.
Similar to~\citet{ajay2020opal}, we relabel the dataset $\mathcal{D}$ in terms of temporally extended transitions using $f_{\psi}^{-1}(a_{0:c-1};s_0)$.
Specifically, we create a dataset $\mathcal{D}_{\text{hi}}=\{(s_0^i,z_i,\sum_{t=0}^{c-1}\gamma^tr_t^{i},s_c^i)\}_{i=1}^{N}$ to learn high-level offline policy, where $z_i\sim f_{\psi}^{-1}(a_{0:c-1};s_0)$.
Next, we adopt the strong offline RL algorithm IQL~\cite{kostrikov2021offline} to learn $\pi_{\phi}(z|s)$:

\begin{figure}[t]
    \centering
    \includegraphics[scale=0.5]{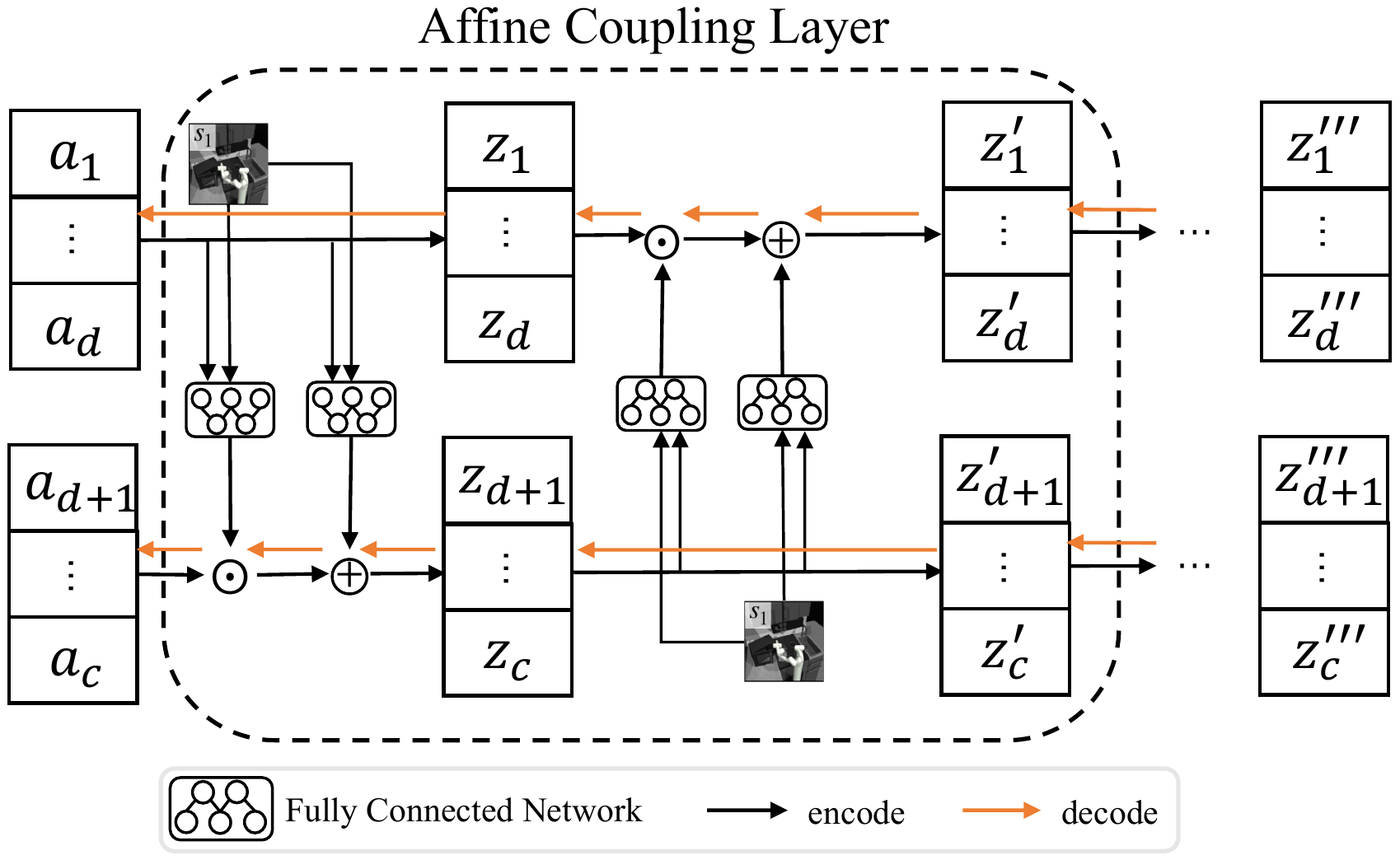}
    \caption{The framework of the flow-based structure. The decoding is $f_{\psi}$ and the encoding is $f^{-1}_{\psi}$.
    The whole process is invertible and lossless.}
    \label{fig:framework_lpd}
\end{figure}

\begin{align*}
    J_V(\varphi) &= \hat{\mathbb{E}}_{\tau\sim\mathcal{D}_{\text{hi}}}\left[L_2^{\lambda}\left(Q_{\vartheta'}(s_0,z) - V_{\varphi}(s_0)\right)\right]\\
    J_Q(\vartheta) &= \hat{\mathbb{E}}_{\tau\sim\mathcal{D}_{\text{hi}}}\left[\left(\sum_{t=0}^{c-1}\gamma^tr_{t} +\gamma V_{\varphi}(s_c)-Q_{\vartheta}(s_0,z)\right)^2\right]\\
    J_{\pi}(\psi) &= \hat{\mathbb{E}}_{\tau\sim\mathcal{D}_{\text{hi}}}\left[\exp(\beta(Q_{\vartheta}(s_0,z) - V_{\varphi}(s_0)))\log\pi_{\phi}(z|s_0)\right]
\end{align*}
where $L_2^{\lambda}(u)=|\lambda-\mathbb{I}(u<0)|u^2$ indicates the expectile regression~\cite{dabney2018distributional}.
$\beta$ is an inverse temperature to control the distribution constraint while maximizing the $Q$-values~\cite{yang2021believe, peng2019advantage}.
Our complete method is described in Algorithm~\ref{HORL_algorithm}.
The graphic overview of our method is shown in Figure~\ref{fig:HORL}.

\begin{algorithm}[tb]
\caption{IQL+LPD algorithm}
\label{HORL_algorithm}
\textbf{Input}: Offline dataset $\mathcal{D}$.\\
\textbf{Parameter}: $\phi$, $\varphi$, $\vartheta$, $\psi$, $w$.\\
\textbf{Output}: policy $\pi_{\phi}$.
\begin{algorithmic}[1]
\STATE Learn $f_{\psi}$ and $\rho_w$ with $\mathcal{D}_{\text{low}}$.
\STATE Create dataset $\mathcal{D}_{\text{hi}}^{r}$ using the trained mapping $f_{\psi}^{-1}$.\\
\WHILE{$t=1$ {\rm \bfseries to} $T$}
\STATE Train value function $Q_{\vartheta}$, $V_{\varphi}$ and policy $\pi_{\phi}$.
\STATE Update target networks:\\
\STATE \qquad $\varphi' \leftarrow (1-\alpha)\varphi' + \alpha\varphi$.\\
\STATE \qquad $\vartheta' \leftarrow (1-\alpha)\vartheta' + \alpha\vartheta$.
\ENDWHILE
\end{algorithmic}
\end{algorithm}

\begin{table*}[t]
    \centering
    \begin{tabular}{l|l|c|c|c|c|c|c}
    \hline
    \textbf{Type} & \textbf{Env} & IQL+LPD & IQL & CQL & OAMPI & TD3+BC & EMAQ \\
    \hline
    partial & kitchen & \textbf{74.9$\pm$1.1} $\uparrow$ & 46.3 & 49.8 & 35.0$\pm$3.3 & 7.5$\pm$1.3 & \textbf{74.6$\pm$0.6} \\
    mixed & kitchen & \textbf{69.2$\pm$1.9} $\uparrow$ & 51.0 & 51.0 & 47.5$\pm$4.1 & 1.5$\pm$0.2 & \textbf{70.8$\pm$2.3}\\
    complete & kitchen & \textbf{75.0$\pm$0.7} $\uparrow$ & 62.5 & 43.8 & 10.0$\pm$1.9 & 23.5$\pm$2.5 & 36.9$\pm$3.7 \\
    \hline
    fixed & Antmaze-umaze & \textbf{93.0$\pm$1.3} $\uparrow$ & 87.5 & 74.0 & 64.3$\pm$4.6 & 78.6$\pm$4.4 & \textbf{91.0$\pm$4.6} \\
    play & Antmaze-medium & \textbf{74.7$\pm$2.2} $\uparrow$ & \textbf{71.2} & 10.6 & 0.0$\pm$0.0 & 33.6$\pm$2.2 & 0.0$\pm$0.0 \\
    play & Antmaze-large & \textbf{56.2$\pm$3.6} $\uparrow$ & 39.6 & 0.2 & 0.3$\pm$0.1 & 21.4$\pm$3.3 & 0.0$\pm$0.0 \\
    diverse & Antmaze-umaze & 81.6$\pm$2.0 $\uparrow$ & 62.2 & 84.0 & 60.7$\pm$3.9 & 71.4$\pm$4.6 & \textbf{94.0$\pm$2.4}\\
    diverse & Antmaze-medium & \textbf{83.7$\pm$1.6} $\uparrow$ & 70.0 & 3.0 & 0.0$\pm$0.0 & 34.7$\pm$2.5 & 0.0$\pm$0.0 \\
    diverse & Antmaze-large & \textbf{52.8$\pm$1.1}$\uparrow$ & 47.5 & 0.0 & 0.0$\pm$0.0 & 25.9$\pm$2.7 & 0.0$\pm$0.0 \\
    \hline
    human & door & \textbf{15.1$\pm$2.5} $\uparrow$ & 4.3 & 9.9 & 2.8$\pm$0.1 & 0.0$\pm$0.0 & - \\
    human & hammer & 3.3$\pm$0.7 $\uparrow$ & 1.4 & \textbf{4.4} & 3.9$\pm$0.2 & 0.9$\pm$0.1 & - \\
    human & pen & 63.1$\pm$1.6 & \textbf{71.5} & 37.5 & 54.6$\pm$4.6 & 39.0$\pm$3.6 & - \\
    cloned & door & \textbf{8.1$\pm$1.0} $\uparrow$ & 1.6 & 0.4 & 0.4$\pm$0.1 & 0.0$\pm$0.0 & 0.2$\pm$0.3 \\
    cloned & hammer & 2.1$\pm$0.2 & 2.1 & 2.1 & 2.1$\pm$0.1 & 0.3$\pm$0.1 & 1.0$\pm$0.7 \\
    cloned & pen & \textbf{65.8$\pm$2.7} $\uparrow$ & 37.3 & 39.2 & 60.0$\pm$5.2 & 25.1$\pm$1.9 & 27.9$\pm$3.7 \\
    \hline
    \end{tabular}
    \caption{
    Performance of IQL+LPD with six baselines on D4RL tasks with the normalized score metric averaged over five random seeds.
    The $\uparrow$ denotes that LPD achieves a performance improvement over IQL.
    The results of IQL, CQL, and EMAQ are taken from the original papers. 
    We implement OAMPI and TD3+BC according to the official code.
    Please refer to Appendix for the detailed implementation of baselines.
    }
    \label{tab:d4rl_results}
\end{table*}

\section{Experiments}
In this section, we aim to address the following questions:
(1) How does the limited representation issue affect the agent's performance, and how does LPD improve on this problem?
(2) How does LPD perform compared to strong baseline offline methods and other skill-based offline methods?
(3) How does the choice of skill length $c$ and other factors affect the overall performance of LPD?
To answer these questions, we conduct extensive experiments on D4RL tasks.
Each experiment result is averaged over five random seeds with a standard deviation.

\subsection{Task Description}
We evaluate our method on a suite of standard and challenging offline tasks~(e.g., D4RL~\cite{fu2020d4rl}) including Franka kitchen, Antmaze, and Adroit. 
Specifically, the Kitchen tasks involve a Franka robot manipulating multiple objects either in an undirected manner or partially task-directed manner.
The task is to use the datasets to arrange objects in the desired configuration, with only a sparse 0-1 completion reward for every object that attains the target configuration.
Antmaze requires composing parts of sub-optimal and undirected trajectories to solve a specific point-to-point navigation problem.
Adroit tasks require controlling a 24-DoF robotic hand to imitate human behavior.
We adopt Kitchen-v0, Antmaze-v0, and Adroit-v0 in our experiments.

\subsection{Baselines}
To answer the first question posed at the start of this section, we compare IQL+LPD against state-of-the-art model-free and model-based offline methods, including Implicit Q-Learning~(IQL)~\cite{kostrikov2021offline}, Conservative Q-Learning~(CQL)~\cite{kumar2020conservative}, Onestep RL(OAMPI)~\cite{brandfonbrener2021offline}, TD3+BC~\cite{fujimoto2021minimalist} and EMAQ~\cite{ghasemipour2021emaq}.
These prior works have achieved superior results on D4RL.

As for the second question, we compare LPD with OPAL~\cite{ajay2020opal}, which is the first study to distill primitives from the offline dataset before applying offline methods to learn a primitive-directing high-level policy.

\begin{table}[t]
    \centering
    \begin{tabular}{|c|c|c|c|}
        \hline
        antmaze~(play) & umaze & medium & large \\
        \hline
        IQL+LPD & \textbf{93.0$\pm$1.3} & \textbf{74.7$\pm$2.2} & \textbf{56.2$\pm$3.6}\\
        \hline
        IQL+OPAL & 83.5$\pm$1.9 & 48.6$\pm$1.0 & \textbf{56.9$\pm$3.3}\\
        \hline
        \hline
        antmaze~(diverse) & umaze & medium & large \\
        \hline
        IQL+LPD & \textbf{81.6$\pm$2.0} & \textbf{83.7$\pm$1.6} & \textbf{52.8$\pm$1.1}\\
        \hline
        IQL+OPAL & 70.2$\pm$1.8 & 42.8$\pm$3.9 & \textbf{52.4$\pm$2.7}\\
        \hline
        \hline
        kitchen & partial & mixed & complete \\
        \hline
        IQL+LPD & \textbf{72.5$\pm$1.1} & \textbf{69.2$\pm$1.9} & \textbf{75.0$\pm$0.7} \\
        \hline
        IQL+OPAL & \textbf{74.9$\pm$0.3} & \textbf{65.7$\pm$3.6} & 11.5$\pm$2.0 \\
        \hline
        \hline
        adroit~(human) & door & hammer & pen \\
        \hline
        IQL+LPD & \textbf{15.1$\pm$2.5} & 3.3$\pm$0.7 & \textbf{63.1$\pm$1.6} \\
        \hline
        IQL+OPAL & 12.1$\pm$2.2 & 1.9$\pm$0.3 & 52.0$\pm$4.6 \\
        \hline
        \hline
        adroit~(cloned) & door & hammer & pen \\
        \hline
        IQL+LPD & \textbf{8.1$\pm$1.0} & 2.1$\pm$0.2 & \textbf{65.8$\pm$2.7} \\
        \hline
        IQL+OPAL & 6.0$\pm$1.0 & 1.1$\pm$0.4 & 46.9$\pm$3.6 \\
        \hline
    \end{tabular}
    \caption{Comparison between LPD and OPAL on D4RL tasks with the normalized score metric averaged over five random seeds.
    }
    \label{tab:horl_opal_d4rl}
\end{table}

\subsection{Main Results}
\paragraph{Experiments with limited representation issue}
We conduct experiments to show the limited representation issue.
A reasonable approach is to evaluate the performance of IQL+OPAL and IQL+LPD with skill length $c=1$.
Experimental results in Figure~\ref{perf_c_1} show a large margin between IQL+OPAL and IQL while the performance of IQL+LPD is consistent with IQL, which indicates the OPAL limits the ability of IQL.
Furthermore, we visualize the similarity $\epsilon$ between the agent's decision $a$ and the dataset, which is calculated by $\epsilon=\min_{\hat{a}\in\mathcal{C}}\|a-\hat{a}\|_{1}$, where $\mathcal{C}=\{(\hat{s},\hat{a})\mid\|s-\hat{s}\|_{1}\leq 20, (\hat{s},\hat{a})\in\mathcal{D}$\} is the similar state set to the agent.
The results in Figure~\ref{fig:generalization_figure} show the decoded action of OPAL is limited to the vicinity of the dataset, while LPD can represent more diverse actions.
This result does not contradict the reconstruction accuracy because we want to decode the action corresponding to the generalized $z$, which does not necessarily exist in the dataset.

\paragraph{Performance on D4RL} 
The experimental results in Table~\ref{tab:d4rl_results} show that IQL+LPD achieves state-of-the-art performance in many tasks and significantly outperforms IQL on nearly all tasks.
Most model-free offline methods such as CQL, OAMPI, and TD3+BC cannot perform well on kitchen and antmaze tasks, which strictly require approximate dynamical programming compared with locomotion tasks.
Although EMAQ is the exiting offline algorithm that achieves good performance in kitchen-partial and kitchen-mixed tasks, IQL+LPD achieves similar or better performance in most kitchen and antmaze tasks.
Moreover, we find TD3+BC and COMBO have poor performance on these challenging tasks, although we have tuned hyper-parameters. 
We suspect this is due to the following reason: while TD3+BC and COMBO are strong in mujoco tasks, they are less effective in more challenging tasks. This is because under complex environments, BC loss is insufficient to preserve conservatism and the learned model is likely to be inaccurate~\cite{kostrikov2021offline}
Finally, we are interested in comparing IQL+LPD and IQL, which directly demonstrates the superiority of the temporally-extended primitives. 
The sources for the scores of the baseline algorithms and the fine-tuning results are shown in Appendix.

\paragraph{Comparison with OPAL}
\label{comparison_OPAL}
We conduct a comparison between LPD and OPAL on D4RL tasks.
Experimental results in Table~\ref{tab:horl_opal_d4rl} show LPD outperforms OPAL in most tasks or achieves the same superior performance.
To ensure a fair comparison, we test the performance of IQL+OPAL with the most suitable steps $c\in\{1, 10\}$ and fine-tune the expectile ratio~$\lambda$ and temperature parameter~$\beta$ in IQL.
Then, we select the best parameter combination of IQL+OPAL.
We reproduce OPAL with authors' providing code via email.
Although the results of OPAL are a little different from the reported results in the original paper, we argue that the results are convincing and comparable for the following reasons:
(1) We incorporate OPAL with IQL. However, this is reasonable since we want the learned primitive policy to be general and can be combined with most offline RL methods.
(2) OPAL has the limited representation issue as the above analysis, which will be discussed in the following ablation study.

\begin{figure}[t]
    \centering
    \includegraphics[scale=0.45]{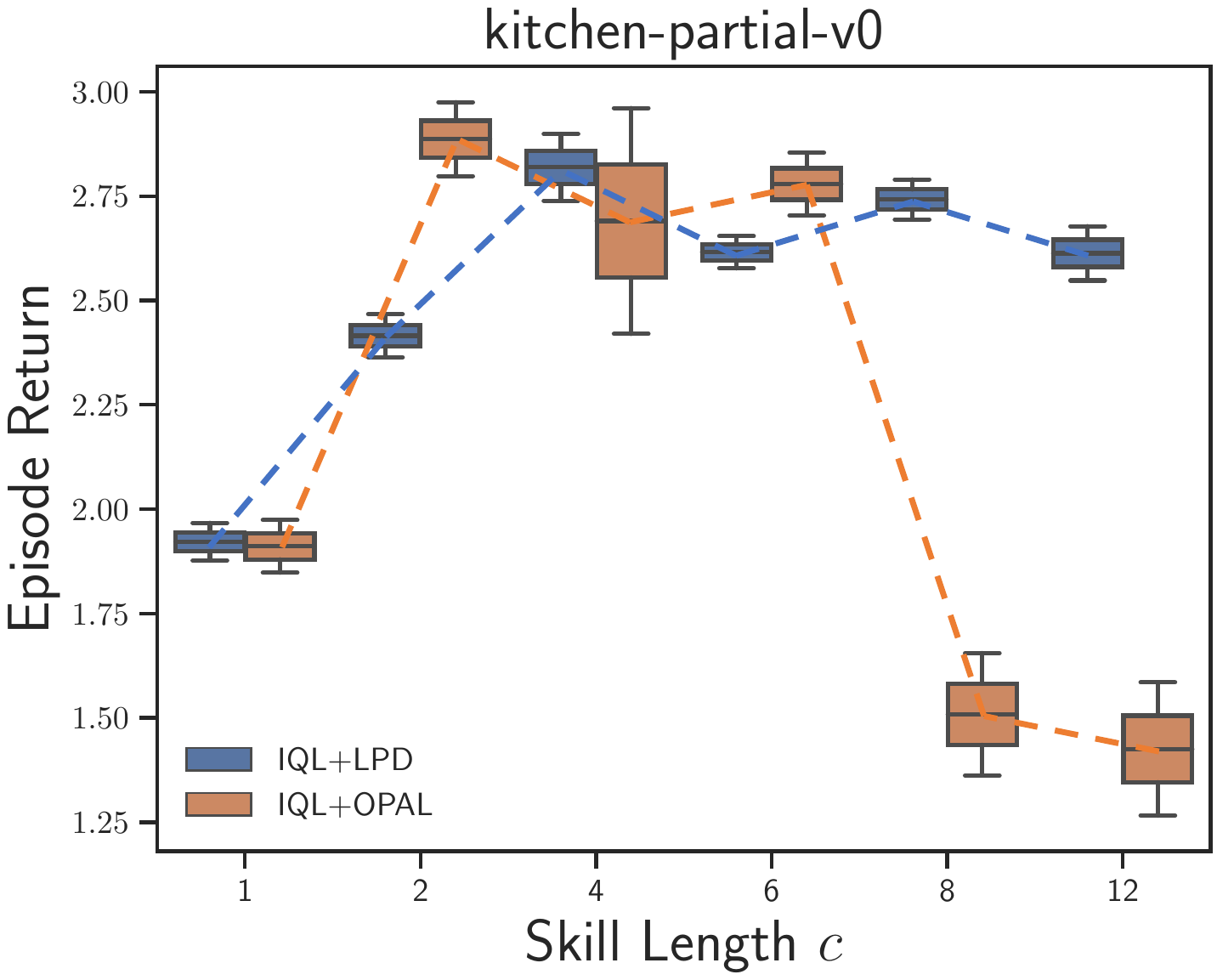}
    \caption{Impact of skill length on performance.}
    \label{fig:skill_step}
\end{figure}

\subsection{Ablation Study}
\paragraph{Impact of skill length on performance}
We conduct experiments on kitchen-partial-v0 to evaluate how final performance is impacted as a result of the skill length.
As one might expect, the results in Figure~\ref{fig:skill_step} show hierarchical offline control achieves better performance than one-level control.
Compared with IQL+OPAL, IQL+LPD can accommodate longer steps.
Furthermore, we find the optimal skill length in this task is around 5.
We also find two methods have similar returns when $c=1$, which does not contradict the above analysis since IQL has limited performance in this task. 
Complete results of the skill length are shown in Appendix.

\paragraph{Impact of the hyper-parameters}
We are concerned about whether the training hyper-parameters of the task policy $\pi_{\phi}$ need to be significantly adjusted on hierarchical offline control.
For this reason, we ran IQL+LPD on kitchen-partial-v0 with various parameters, such as the expectile ratio $\lambda\in[0.45, 0.9]$ and the temperature $\beta\in[0.35, 0.8]$.
The experimental results in Figure~\ref{fig:framework_ldp} show that the performance of IQL+LPD is robust to the changes of the hyper-parameters.
Furthermore, we only need to finetune $\beta$ around the original parameter without tuning $\lambda$.

\begin{figure}[t]
    \centering
    \includegraphics[scale=0.55]{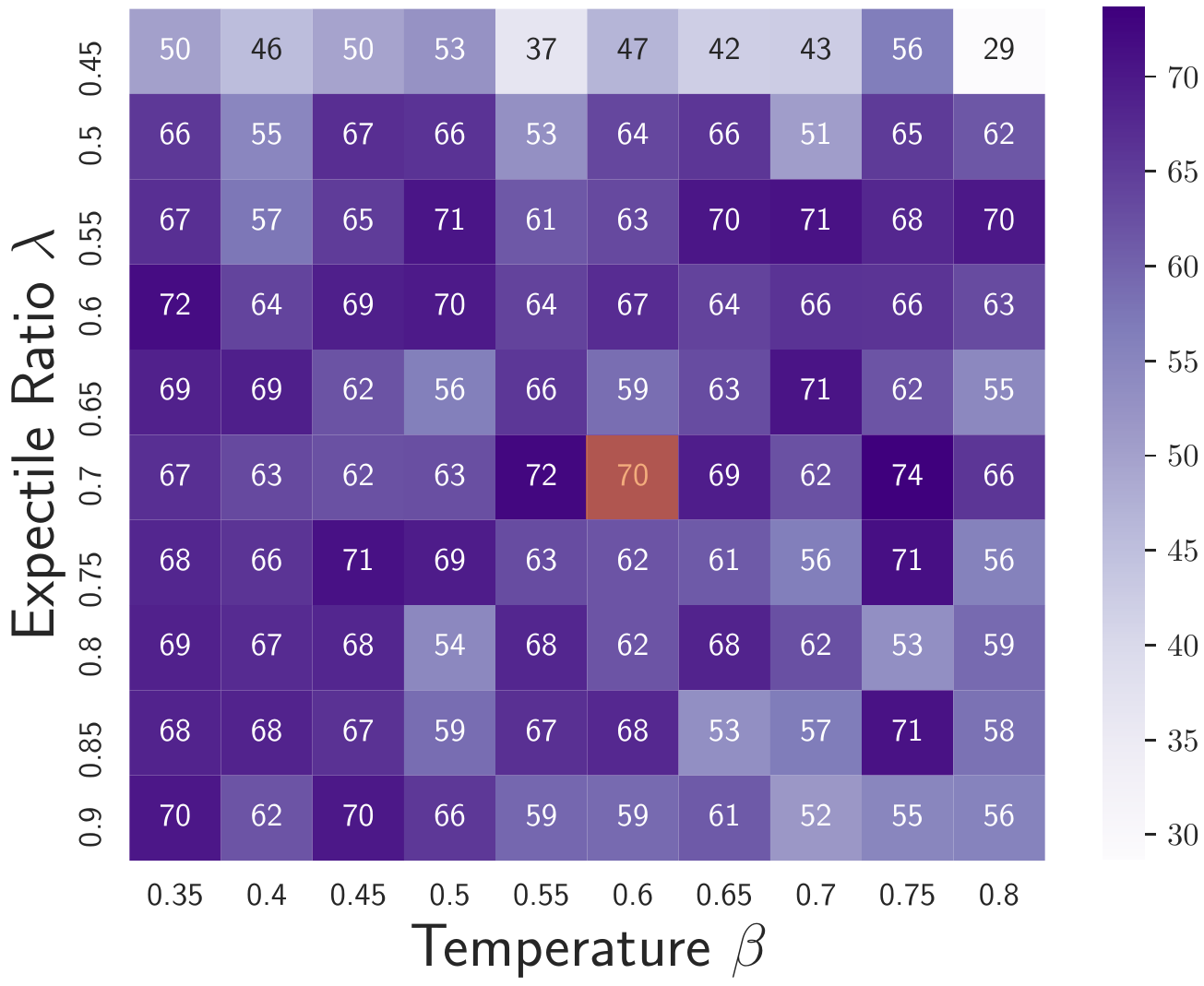}
    \caption{Impact of the hyper-parameters.
    We evaluate IQL+LPD in kitchen-partial-v0 task with various parameter $\lambda\in[0.45,0.9]$ and $\beta\in[0.35,0.8]$.
    The orange square is the default parameter of IQL and we adopt the normalized score metric.
    }
    \label{fig:framework_ldp}
\end{figure}

\section{Conclusion}
In this paper, we show that there are provable benefits in learning a hierarchical structure with offline datasets, and it is crucial for low-level primitive skills to be faithful to the original behavior and be expressive to recover the original state space.
We empirically show that current skill-based offline RL methods are significantly compromised by the limited representation ability of the learned low-level skills.
To solve this issue, we propose the offline Lossless Primitive Discovery (LPD), which learns an invertible function between latent vectors and temporally extended actions.
We show that our proposed method has a powerful representation ability to recover the original policy space and it achieves strong results on various D4RL tasks.
Learning such expressive skills also enables offline few-show learning for downstream tasks, and we leave this as an interesting future direction.

\clearpage
\section{Acknowledgments}
This work was funded by the National Natural Science Foundation of China (ID:U1813216, 62192751 and 61425024), National Key Research and Development Project of China under Grant 2017YFC0704100 and Grant 2016YFB0901900, in part by the 111 International Collaboration Program of China under Grant BP2018006, BNRist Program (BNR2019TD01009) and the National Innovation Center of High Speed Train R\&D project (CX/KJ 2020-0006), Science and Technology Innovation 2030 - “New Generation Artificial Intelligence” Major Project (No. 2018AAA0100904) and National Natural Science Foundation of China (62176135).

\bibliography{aaai23}

\clearpage
\appendix

\def\dtv{D_{\mathrm{TV}}}
\def\Wbar{\overline{w}}
\def\deriv{~d}
\section{Offline Algorithms}
\label{offline_algorithm}
In this subsection, we give an detailed algorithm for provable offline learning in linear MDPs. Specifically, we consider the \textit{pessimistic value iteration} \citep[PEVI;][]{jin2021pessimism}, as shown in Algorithm~\ref{alg:1}, which uses uncertainty as a negative bonus for value learning.
PEVI uses negative bonus $\Gamma(\cdot,\cdot)$ over standard $Q$-value estimation $\hat{Q}(\cdot,\cdot) =  (\hat\BB \hat{V})(\cdot)$ to reduce potential bias due to finite data, where $\hat\BB$ is the empirical estimation of $\BB$ from dataset $\cD$. We use the notion of  $\xi$-uncertainty quantifier as follows to formalize the idea of pessimism.
\begin{definition}[$\xi$-Uncertainty Quantifier] 
    We say $\Gamma :\cS\times\cA\to \RR$ is a $\xi$-uncertainty quantifier for $\hat\BB$ and $\widehat{V}$ if with probability $1-\xi$,
    \begin{equation}
     \big|(\hat\BB\hat{V})(s,a) - (\BB\hat{V})(s,a)\big|\leq \Gamma(s,a), 
    \label{eq:def_event_eval_err_general}
    \end{equation}
    for all~$(s,a)\in \cS\times \cA$.
    \label{def:uncertainty_quantifier}
\end{definition}

In linear MDPs, we can construct $\hat\BB\hat{V}$ and $\Gamma$ based on $\cD$ as follows, where $\hat\BB\hat{V}$ is the empirical estimation for $\BB\hat{V}$. For a given dataset $\cD=\{(s_\tau,a_\tau,r_\tau)\}_{\tau=1}^{N}$, we define the empirical mean squared Bellman error (MSBE) as
\begin{equation*}
M(w) = \sum_{\tau=1}^N \bigl(r_\tau + \gamma \widehat{V}(s_{\tau+1}) - \phi (s_\tau,a_\tau)^\top w\bigr)^2 + \lambda \norm{w}_2^2
\end{equation*}
Here $\lambda>0$ is the regularization parameter. Note that $\hat{w}$ has the closed form
\#\label{eq:w18}
&\hat{w} =  \Lambda ^{-1} \Big( \sum_{\tau=1}^{N} \Phi(s_\tau,a_\tau) \cdot \bigl(r_\tau + \gamma\hat{V}(s_{\tau+1})\bigr) \Bigr ) , \notag\\
&\text{where~~} \Lambda = \lambda I+\sum_{\tau=1}^N \Phi(s_\tau,a_\tau)  \Phi(s_\tau,a_\tau) ^\top. 
\#
Then we let $\hat\BB\hat{V}=\langle\phi,\hat w  \rangle$. Meanwhile, we construct $\Gamma$ based on $\cD$ as 
\#\label{eq:w05}
\Gamma(s, a) = \beta\cdot \big( \Phi(s, a)^\top  \Lambda ^{-1} \Phi(s, a)  \big)^{1/2}.
\#
Here $\beta>0$ is the scaling parameter.

\begin{algorithm}[H]
    \caption{Pessimistic Value Iteration}\label{alg:1}
    \begin{algorithmic}[1]
    \STATE {\bf Require}: Dataset $\cD_{\text{hi}}=\{(s_\tau,a_\tau,\sum_{i=0}^{c-1}\gamma^i r_{\tau+i}^{i})\}_{\tau=1}^{T}$, low-level skills $\pi_\theta$.
    \STATE Initialization: Set $\hat{V}(\cdot) \leftarrow 0$ and construct $\Gamma(\cdot, \cdot)$.
    \WHILE{not converged}
    \STATE Construct $(\hat\BB \hat{V})(\cdot,\cdot)$
    \STATE Set $\hat{Q}(\cdot,\cdot) \leftarrow  (\hat\BB \hat{V})(\cdot,\cdot)- \Gamma(\cdot,\cdot)$.
    \STATE Set $\hat{\pi} (\cdot \given \cdot) \leftarrow \argmax_{\pi}\EE_{\pi}{\left[\hat{Q}(\cdot, \cdot)\right]}$.
    \STATE Set $\hat{V}(\cdot) \leftarrow  \EE_{\hat{\pi}}{\left[\hat{Q}(\cdot, \cdot)\right]}$. \label{alg:general_Vhat}
    \ENDWHILE
    \STATE \textbf{Return} $\hat\pi_\theta$%
    \end{algorithmic}
\end{algorithm}

\section{Missing Proofs}
\label{proofs}
Inspired by \citet{nachum2018near}, we first prove the following lemma.
\begin{lemma}
    \label{lemma:subopt-temp}
    We define the transition error $\epsilon_k:S\times\Pi\to\RR$
    for $k\in[1, c]$ as, 
    \begin{equation}
        \epsilon_k(s_t) = \dtv(P_1(s_{t+k}|s_{t+k-1}) || P_2(s_{t+k}|s_{t+k-1})).
    \end{equation}
    If
    \begin{equation}
    \label{eq:dtv-cond-temp}
    \sup_{s_t\in S, k\in[1,c]} \EE_{s\sim d_1}[\epsilon_k(s)] \le \epsilon, 
    \end{equation}
    then we have
    \begin{equation}
        |J(\pi,M_1) - J(\pi,M_2)| \le \frac{\gamma c(c+1)r_{max}}{(1-\gamma^c)(1-\gamma)} \epsilon.
    \end{equation}

\end{lemma}
\begin{proof}

Denote the $k$-step state transition distributions under either $P_1$ or $P_2$ as,
\begin{align}
    P_1^k(s'|s) &= P_1(s_{t+k}=s'|s_t=s), \notag\\
    P_2^k(s'|s) &= P_2(s_{t+k}=s'|s_t=s),
\end{align}
for $k\in[1,c]$.
Considering $P_1,P_2$ as linear operators, we may express the state visitation frequencies $d_1,d_2$ of $P_1, P_2$, respectively, as
\begin{align}
    d_1 &= (1-\gamma) A_1 (I - \gamma^c P_1^c)^{-1} \mu, \notag\\
    d_2 &= (1-\gamma) A_2 (I - \gamma^c P_2^c)^{-1} \mu,
\end{align}
where $\mu$ is a Dirac $\delta$ distribution centered at $s_0$ and
\begin{align}
    A_1 &= I + \sum_{k=1}^{c-1} \gamma^k P_1^k,\notag\\
    A_2 &= I + \sum_{k=1}^{c-1} \gamma^k P_2^k.
\end{align}
We will use $d_1^c,d_2^c$ to denote the every-$c$-steps $\gamma$-discounted state frequencies of $P_1, P_2$; i.e.,
\begin{align}
    d_1^c &= (1-\gamma^c) (I - \gamma^c P_1^c)^{-1} \mu, \notag\\
    d_2^c &= (1-\gamma^c) (I - \gamma^c P_2^c)^{-1} \mu.
\end{align}
By the triangle inequality, we have the following bound on the total variation divergence $|d_2 - d_1|$:
\begin{multline}
    \label{eq:d-tri}
    |d_2 - d_1| \le (1-\gamma)|A_2(I-\gamma^cP_2^c)^{-1}\mu - A_2(I-\gamma^cP_1^c)^{-1}\mu| \\+
    (1-\gamma)|A_2(I-\gamma^cP_1^c)^{-1}\mu - A_1(I-\gamma^cP_1^c)^{-1}\mu|.
\end{multline}
We begin by attacking the first term of Equation~\eqref{eq:d-tri}. We note that
\begin{equation}
    |A_2| \le |I| + \sum_{k=1}^{c-1}\gamma^k |P_2^k| = \frac{1-\gamma^c}{1-\gamma}.
\end{equation}
Thus the first term of Equation~\eqref{eq:d-tri} is bounded by
\begin{multline}
    \label{eq:long-eqn}
    (1-\gamma^c)|(I-\gamma^cP_2^c)^{-1}\mu - (I-\gamma^cP_1^c)^{-1}\mu| \\
    = (1-\gamma^c)|(I-\gamma^cP_2^c)^{-1}((I - \gamma^cP_1^c) - (I - \gamma^cP_2^c)) (I-\gamma^cP_1^c)^{-1}\mu|  \\
    = \gamma^c|(I-\gamma^cP_2^c)^{-1}(P_2^c - P_1^c) d_1^c|.
\end{multline}
By expressing $(I-\gamma^cP_2^c)^{-1}$ as a geometric series and employing the triangle inequality, we have $|(I-\gamma^cP_2^c)^{-1}|\le (1-\gamma^c)^{-1}$, and we thus bound the whole quantity (\eqref{eq:long-eqn}) by
\begin{equation}
    \label{eq:first-term}
    \gamma^c(1-\gamma^c)^{-1}|(P_2^c - P_1^c)d_1^c|.
\end{equation}

We now move to attack the second term of Equation~\eqref{eq:d-tri}.
We may express this term as
\begin{equation}
    (1-\gamma)(1-\gamma^c)^{-1}|(A_2 - A_1)d_1^c|.
\end{equation}
Furthermore, by the triangle inequality, we have
\begin{equation}
    \label{eq:second-term}
    |(A_2 - A_1)d_1^c| \le \sum_{k=1}^{c-1} \gamma^k |(P_2^k - P_1^k)d_1^c|.
\end{equation}
Combining Equation~\eqref{eq:first-term} and \eqref{eq:second-term}, we have
\begin{align}
    |d_2 - d_1| \leq &\gamma(1-\gamma^c)^{-1}\sum_{k=1}^{c} w_k |(P_2^k - P_1^k)d_1^c| \notag\\
    =&2\gamma(1-\gamma^c)^{-1}\cdot \notag\\
    &\sum_{k=1}^c \E_{s\sim d_1^c}[w_k \dtv(P_1^k(s'|s) || P_2^k(s'|s))] \notag\\
    \leq & 2\gamma(1-\gamma^c)^{-1}\sum_{k=1}^c  k \epsilon\notag\\
    \leq & \frac{\gamma c(c+1)}{1-\gamma^c} \epsilon,
\end{align}

where $w_k = \gamma^{k-1}(1-\gamma)$ if $k<c$ and $w_k=\gamma^{k-1}$ if $k=c$. The second inequality uses the fact that $w_k\leq 1$ for all $k\in[1,c]$.

We now move to consider the difference in values.  We have
\begin{align}
    V^{1}(s_0) &= (1-\gamma)^{-1}\int_{S} d_1(s) R(s) \deriv s, \\
    V^{2}(s_0) &= (1-\gamma)^{-1}\int_{S} d_2(s) R(s) \deriv s.
\end{align}
Therefore, we have
\begin{align}
    |V^{1}(s_0) - V^{2}(s_0)| &\le (1-\gamma)^{-1} r_{max} |d_2 - d_1| \\
    &\le \frac{c(c+1) }{(1-\gamma^c)(1-\gamma)} r_{max} \epsilon,
\end{align}
as desired.
\end{proof}
\subsection{Proof of Proposition~\ref{prop:1}}
\begin{proof}
    The $1$-step transition probability can be writen as 
\begin{align*}
    P_{\beta}(s_{t+1}|s_{t},z) &= \int{P(s_{t+1}|s_t,a_t)\beta(a_t| s_t,z)da_t} \\
    &= \int{\beta(a_t| s_t,z)\Psi(s_t,a_t,s_{t+1})^T\omega d a_t} \\
    &\doteq \Psi_1(s_t,a_t,s_{t+1})^T \omega .
\end{align*}
It is easy to see that $\|\Psi_1\|_\infty\leq 1$.
The $2$-step transition can be similarly writen as 
\begin{align*}
    &P_{\beta}(s_{t+2}|s_{t},z) \\
    &=  \int{\Psi_1(t)^T\omega \Psi_1(t+1)^T\omega ds_{t+1}} \\
    &\doteq \Psi_2^T(s_t,a_t,s_{t+2}) \omega,
\end{align*}
where $\Psi_1(t) = \Psi_1(s_t,a_t,s_{t+1})$ and $\Psi_1(t+1) = \Psi_1(s_{t+1},a_{t+1},s_{t+2})$.
Since $\Psi_1(s_t,a_t,s_{t+1})^T\omega$ is a probability over $s_{t+1}$, we also have $\|\Psi_2\|_\infty\leq 1$.
Thus the $c$-step transition probability can be written in the form $P_{\beta}(s_{t+c}|s_{t},z) = \Psi_c(s_t,a_t,s_{t+c})^T \omega $ and we have $\|\Psi_c\|_\infty\leq 1$.
For the reward function, we have 
\begin{equation}
    r_{\beta,c}(s_t,z) = \EE_{\beta}\left[\sum_{k=0}^{c-1}{\gamma^c \Phi(s_{t+k},a_{t+k})^T \omega}\right] \doteq \Phi_c(s_t,z)^T \omega,
\end{equation}
which is also linear with respect to $\omega$. And we have $r_{\text{max},c} = \sum_{k=0}^{c-1}{\gamma^c r_{\text{max}}} = \frac{1-\gamma^c}{1-\gamma} r_{\text{max}}$.
\end{proof}

\subsection{Proof of Lemma~\ref{lemma:1}}
Following Lemma 8 in \citet{uehara2021pessimistic}, we have the following guarantee for the low-level skill learning with the MLE objective:

\begin{equation}
    \EE_{\beta} [\dtv (\pi_\theta(\cdot|s,z)||\beta(\cdot|s,z))] \leq \varepsilon_{\theta} = \sqrt{\frac{\ln{|\Pi_\theta|/{\delta}}}{N}}
\end{equation}

Note that 
$P_{\pi,\beta}(\cdot|s) = \int \pi(z|s)\beta(\cdot|s,z) dz $ and $P_{\pi,\theta}(\cdot|s) = \int \pi(z|s)\pi_\theta(\cdot|s,z) dz $,
we also have 
\begin{equation}
    \EE_{\beta} [\dtv (P_{\pi,\beta}(\cdot|s)||P_{\pi,\theta}(\cdot|s))] \leq \varepsilon_{\theta},
\end{equation}
then the result follows immediately from Lemma~\ref{lemma:subopt-temp}.
\subsection{Proof of Lemma~\ref{lemma:2}}
Here we first provide a formal statement of Lemma~\ref{lemma:2}.
\begin{lemma}[Formal]
    \label{lemma:2_formal}
    Suppose there exists an absolute constant 
    \begin{align}
        \label{eq:event_opt_explore}
        c^\dagger = &\sup_{x\in \RR^d} \frac{x^{\top}\Sigma_{\pi^*,s}x}{x^{\top}\Sigma_{\cD} x} < \infty,
    \end{align}
    for all $s\in\cS$ with probability $1-\delta/2$, where 
    \$&\Sigma_{\cD}~~~=\frac{1}{N}\sum_{\tau=1}^N{\left[\Phi(s_\tau,a_\tau)\Phi(s_\tau,a_\tau)^\top\right]},\notag\\
     &\Sigma_{\pi^*,s}=\EE_{{\pi^*}}\bigl[\Phi(s_\tau,a_\tau)\Phi(s_\tau,a_\tau)^\top\biggiven s_0=s\bigr].
     \$

    In Algorithm 2, we set 
    \begin{equation*}
        \lambda =1,~ \beta= C \cdot d r_{\text{\rm max}}\sqrt{\zeta}/(1-\gamma^c), ~ \zeta = \log{(4dN/(1-\gamma^c)\delta)},
    \end{equation*}
    where $C>0$ is an absolute constant  and $\delta \in (0,1)$ is the confidence parameter. Then with probability $1-\delta$, the policy $\widehat{\pi}$ generated by Algorithm 2 satisfies
    \begin{align*}\label{eq:event_opt_explore_d_2}
        J(\pi^*_{\beta})-J(\widehat{\pi}_{\beta})
    \leq  \frac{2C r_{\text{\rm max}}}{(1-\gamma)(1-\gamma^c)} \sqrt{c^\dagger d^3\zeta /N},  \forall s\in \cS.
    \end{align*}
\end{lemma}
\begin{proof}
    From Proposition~\ref{prop:1} we can see that the hyper-MDP is also an linear MDP with $r_{\text{max}}' = \frac{1-\gamma^c}{1-\gamma} r_{\text{max}}$. Following the similar argument in Lemma 3.1 in ~\citet{hu2022role}, we have 
    \begin{align*}
        J(\pi^*_{\beta})-J(\widehat{\pi}_{\beta})
    &\leq  \frac{2C r'_{\text{\rm max}}}{(1-\gamma^c)^2} \sqrt{c^\dagger d^3\zeta /N} \\
    &= \frac{2C r_{\text{\rm max}}}{(1-\gamma)(1-\gamma^c)} \sqrt{c^\dagger d^3\zeta /N},  \forall s\in \cS.
    \end{align*}
\end{proof}



\subsection{Proof of Lemma~\ref{lemma:3}}
\begin{proof}
    Let $P_1 = P_\beta$ and $P_2 = P_{\tilde{\beta}}$ and the result follows immediately from Lemma~\ref{lemma:subopt-temp}.
\end{proof}

\subsection{Proof of Theorem~\ref{theorem:1}}
\begin{proof}
From the definition, we have 
\begin{align*}
    &\text{\rm SubOpt}(\widehat{\pi}_{\theta}) =   J(\pi^*)-J(\widehat{\pi}_{\theta}) \\
    &= {J(\widehat{\pi}_{\beta})-J(\widehat{\pi}_{\theta})}+J(\pi^*_{\beta})-J(\widehat{\pi}_{\beta}) +J(\pi^*)-J(\pi^*_{\beta}).
\end{align*}

From Lemma~\ref{lemma:1}, \ref{lemma:2}, \ref{lemma:3} and the union bound, we have the result immediately.
\end{proof}


\clearpage

\section{Implementation Details}
\subsection{Baseline details}
In this section, we provide details about baselines we ran ourselves.
For scores of baselines previously evaluated on standardized tasks, we provide the source of the listed score.

\paragraph{IQL} The performance of IQL on the D4RL tasks is taken from~\cite{kostrikov2021offline}.
\paragraph{CQL} The performance of CQL on the D4RL tasks is taken from~\cite{kumar2020conservative}.
\paragraph{OAMPI} The performance of OAMPI on the adroit tasks is taken from~\cite{brandfonbrener2021offline}.
We ran OAMPI using official implementation on the kitchen and antmaze tasks from the authors.
\paragraph{TD3+BC} We ran TD3+BC using official implementation on D4RL tasks from the authors.
We tuned over the hyperparameter: $\alpha\in[0.5, 4.5]$.
\paragraph{EMAQ} The performance of EMAQ on the D4RL tasks is taken from~\cite{ghasemipour2021emaq}.
\paragraph{COMBO} We ran COMBO using official implementation on D4RL tasks from the authors.

\subsection{OPAL experiment details}
We used the code provided by the author via email.
Specifically, the encoder $q_{\psi}(z|\tau)$ takes in state-action trajectory $\tau$ of length $c$ and outputs the mean and log standard deviation of the latent vector.
The prior $\rho_{w}(z|s)$ takes in the current state $s$ and outputs the latent vector's mean and log standard deviation.
The primitive policy $\pi_{\theta}(a|s,z)$ has the same architecture as the prior, but it takes in the state and latent vector and produces the mean and log standard deviation of the action.
The framework of OPAL is shown in Figure~\ref{fig:framework_opal}.
\begin{figure}[h]
    \centering
    \includegraphics[scale=0.6]{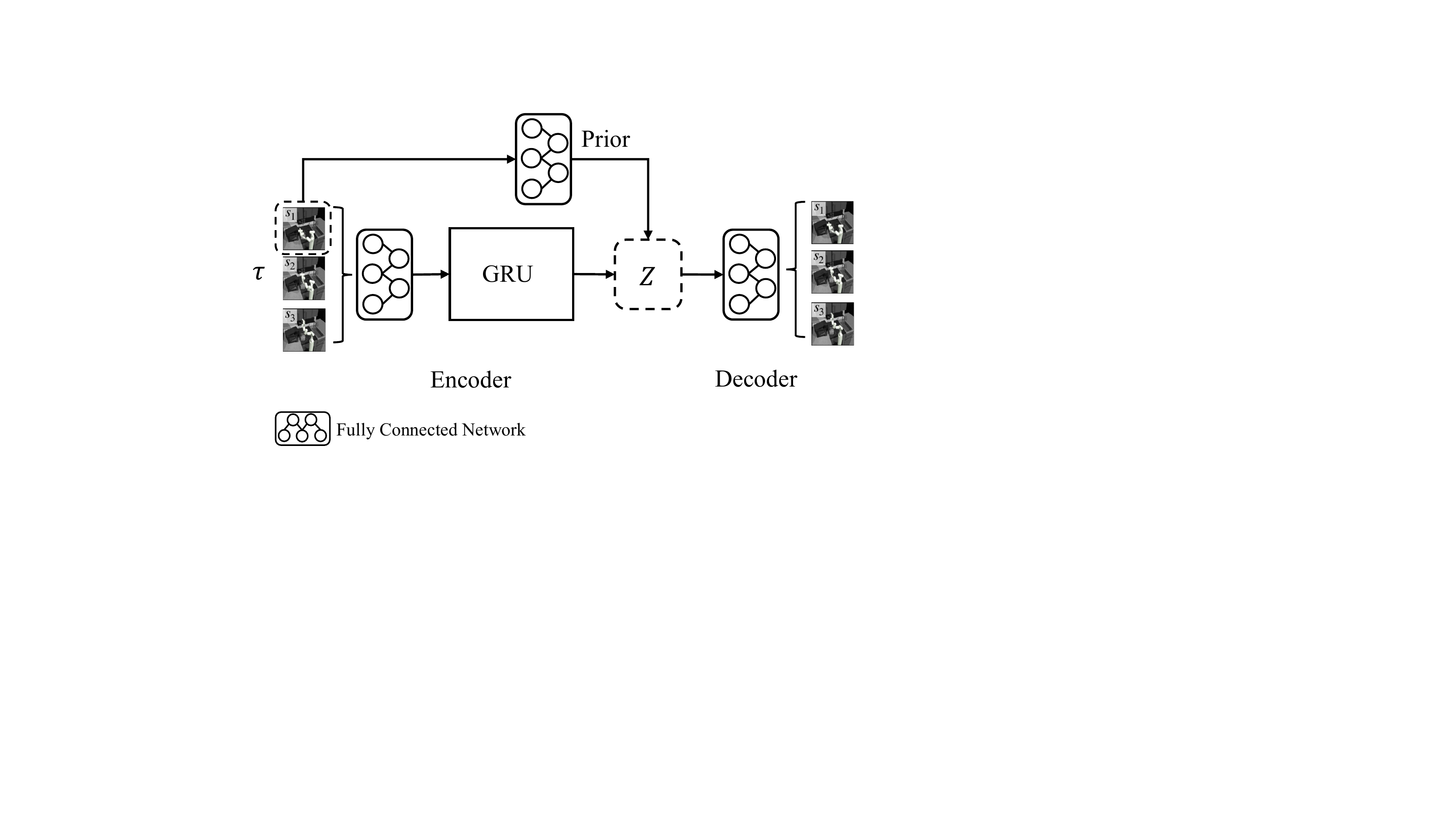}
    \caption{The framework of OPAL.}
    \label{fig:framework_opal}
\end{figure}
We adopt the same network structure parameter as the reported paper.
For, example, the hidden layer size $H=200$ in antmaze and adroit tasks, and $H=256$ in kitchen tasks.

As for the task policy IQL, we adopt the official implementation and the same architecture as LPD.
Then we finetune the expectile ratio~$\lambda$ and temperature~$\beta$.

In antmaze tasks, we use $c=10$.
In the kitchen and adroit tasks, we use $c=5$ since it works better. 

\subsection{LPD experiment details}
We adopt the conditional real NVP with several affine coupling layers as the LDP's structure, which is shown in Figure~\ref{fig:framework_ldp}.
\begin{figure}[h]
    \centering
    \includegraphics[scale=0.5]{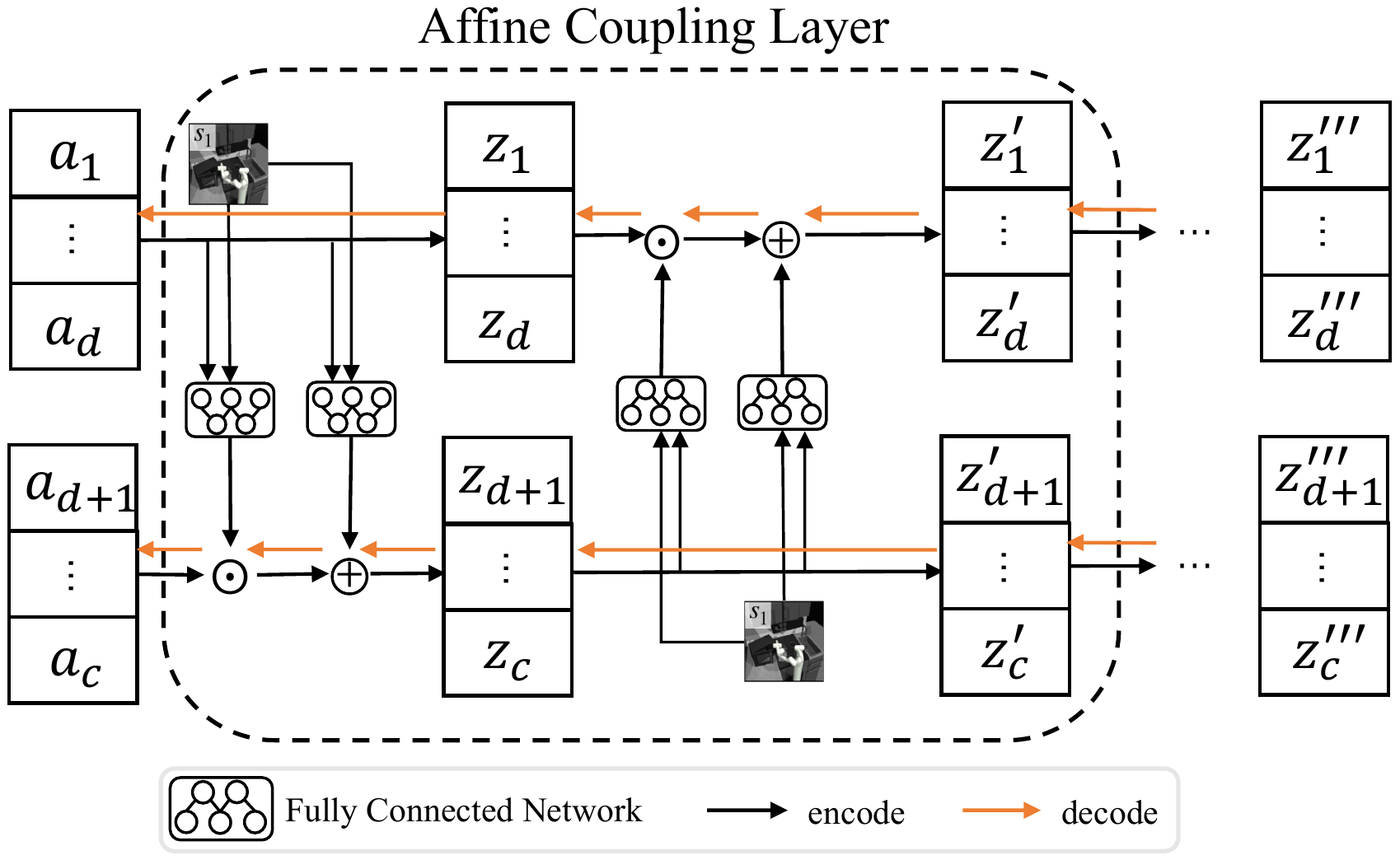}
    \caption{The framework of LPD.}
    \label{fig:framework_ldp}
\end{figure}
Specifically, given the input $a_{1:c}$, the affine coupling layer transforms it into $z$ through the operation: $z_{1:d}=a_{1:d}$ and $z_{d+1:c}=a_{d+1:c}\bigodot\exp(v(a_{1:d};s_1)) + t(a_{1:d};s_1)$, where $v$ and $t$ are implemented using the fully connected network.
Next, we flip the output and conduct the same operation again.
We repeat the affine coupling layer $k$ times in experiments with $H$ hidden dim in Fully connected networks.

In antmaze tasks, we select $k=1$ and $H=50$.
In kitchen tasks, we select $k=3$ and $H=100$.
In Adroit-human tasks, we select $k=1$ and $H=256$.
As for the skill length $c$, we select $c=4$ in kitchen and adroit tasks.
We select $c=2$ in antmaze tasks since we find it is sufficient for good performance.
The hyper-parameters for IQL+OPAL and IQL+LPD are the same as IQL other than finetuning the temperature~$\beta$. 

\end{document}